\documentclass[cleveref,final,12pt]{colt2024} %

\usepackage{enumitem}
\usepackage{subcaption}
\usepackage{booktabs}       %
\usepackage{multirow}
\usepackage{graphicx}
\usepackage{pifont}

\usepackage{amsmath,amsfonts,bm,dsfont,mathtools}

\def\<#1,#2>{\left\langle #1,#2 \right\rangle}

\newcommand{\norm}[1]{\|#1\|}

\newcommand{\abs}[1]{|#1|}

\def\eqref#1{equation~\ref{#1}}

\def\rz{{\textnormal{z}}}

\def\rvz{{\mathbf{z}}}

\def\rmA{{\mathbf{A}}}

\def\rmZ{{\mathbf{Z}}}

\def\mA{{\bm{A}}}

\def\mI{{\bm{I}}}

\def\mX{{\bm{X}}}

\def\mSigma{{\bm{\Sigma}}}

\DeclareMathAlphabet{\mathsfit}{\encodingdefault}{\sfdefault}{m}{sl}
\SetMathAlphabet{\mathsfit}{bold}{\encodingdefault}{\sfdefault}{bx}{n}

\newcommand{\prob}{\mathbb{P}}
\newcommand{\E}[2][]{\mathbb{E}_{#1}\left[#2\right]} %

\newcommand{\R}{\mathbb{R}}

\newcommand{\var}[2][]{{ \operatorname{Var}_{#1}\left(#2\right)}}

\DeclareMathOperator{\tr}{\mathrm{tr}}
\newcommand{\mat}[1]{#1}
\DeclareMathOperator{\diag}{\mat{diag}}

\usepackage[breakable]{tcolorbox}
\definecolor{block-gray}{gray}{0.9}
\newtcolorbox{blockquote-orange}{colback=orange!15!white,grow to right by=-1mm,grow to left by=-1mm,boxrule=0pt,boxsep=0pt,breakable}
\newtcolorbox{blockquote-grey}{colback=block-gray,grow to right by=-1mm,grow to left by=-1mm,boxrule=0pt,boxsep=0pt,breakable}

\title[Analysis of JL: Unified and simple]{Simple, unified analysis of Johnson-Lindenstrauss with applications}
\usepackage{times}

\coltauthor{%
 \Name{Yingru Li\thanks{The author would like to acknowledge Professor Zhi-Quan (Tom) Luo for advising this project.} } \Email{yingruli@link.cuhk.edu.cn}\\
 \addr The Chinese University of Hong Kong, Shenzhen, China
}

\begin{document}

\maketitle

\begin{abstract}
    We present a simplified and unified analysis of the Johnson-Lindenstrauss (JL) lemma, a cornerstone of dimensionality reduction for managing high-dimensional data. Our approach simplifies understanding and unifies various constructions under the JL framework, including spherical, binary-coin, sparse JL, Gaussian, and sub-Gaussian models. This unification preserves the intrinsic geometry of data, essential for applications from streaming algorithms to reinforcement learning. We provide the first rigorous proof of the spherical construction's effectiveness and introduce a general class of sub-Gaussian constructions within this simplified framework. Central to our contribution is an innovative extension of the Hanson-Wright inequality to high dimensions, complete with explicit constants. By using simple yet powerful probabilistic tools and analytical techniques, such as an enhanced diagonalization process, our analysis solidifies the theoretical foundation of the JL lemma by removing an independence assumption and extends its practical applicability to contemporary algorithms.
\end{abstract}

\begin{keywords}%
    Dimensionality reduction, Johnson-Lindenstrauss, Hanson-Wright, Matrix factorization, Uncertainty estimation, Epistemic Neural Networks (ENN), Hypermodel
\end{keywords}

\section{Introduction}
\label{sec:intro}
In the realm of modern computational algorithms, dealing with high-dimensional data often necessitates a preliminary step of dimensionality reduction.
This process is not merely a matter of convenience but a critical operation that preserves the intrinsic geometry of the data. Such dimensionality reduction techniques find widespread application across a diverse array of fields, including but not limited to streaming algorithms \citep{muthukrishnan2005data}, compressed sensing \citep{candes2006near,baraniuk2008simple}, numerical linear algebra \citep{woodruff2014sketching}, feature hashing \citep{weinberger2009feature}, uncertainty estimation \citep{li2022hyperdqn,osband2023epistemic} and reinforcement learning \citep{li2022hyperdqn,li2024hyperagent}. These applications underscore the technique's versatility and its fundamental role in enhancing algorithmic efficiency.

The essence of geometry preservation within the context of dimensionality reduction can be mathematically formulated as the challenge of designing a probability distribution over matrices that effectively retains the norm of any vector within a specified error margin after transformation. Specifically, for a given vector \(x \in \mathbb{R}^n\), the objective is to ensure that with probability at least $1 - \delta$, the norm of \(x\) after transformation by a matrix \(\Pi \in \R^{m \times n}\) drawn from the distribution \(\mathcal{D}_{\varepsilon, \delta}\) remains  \(\epsilon\)-approximation of its original norm, as shown below:
\begin{align}
\label{eq:geometry-preserve}
    \underset{\Pi \sim \mathcal{D}_{\varepsilon, \delta}}{\mathbb{P}}\left(\|\Pi x\|_2^2 \in \left[(1-\varepsilon)\|x\|_2^2,(1+\varepsilon)\|x\|_2^2\right]\right) \ge 1 - \delta
\end{align}
A foundational result in this domain, the following Johnson-Lindenstrauss (JL) lemma, establishes a theoretical upper bound on the reduced dimension \(m\), achievable while adhering to the above-prescribed fidelity criterion.
\begin{lemma}[JL lemma \citep{johnson1984extensions}]
  \label{lem:djl}
  For any $0<\varepsilon, \delta<1 / 2$, there exists a distribution $\mathcal{D}_{\varepsilon, \delta}$ on $\mathbb{R}^{m \times n}$ for $m=O(\varepsilon^{-2} \log (1 / \delta))$ that satisfies \cref{eq:geometry-preserve}. 
\end{lemma}
Recent research \citep{kane2011almost,jayram2013optimal} has validated the optimality of the dimension \(m\) specified by this lemma, further cementing its significance in the field of dimensionality reduction.

Initially, the constructive proof for \Cref{lem:djl} is based on the random $k$-dimensional subspace \citep{johnson1984extensions,frankl1988johnson,dasgupta2003elementary}. Projection to a random subspace involves computing a random rotation matrix, which requires computational-intensive orthogonalization processes.
Along the decades, many alternative JL distributions $\mathcal{D}_{\varepsilon, \delta}$ were developed for the convenience of computation and storage.
\citet{indyk1998approximate} chooses the entries of $\Pi$ as independent Gaussian random varaibles, i.e. $\Pi \sim \frac{1}{\sqrt{m}} \cdot N(0, 1)^{\otimes (m \times n)}$ where the random matrix is easier and faster to generate by skipping the orthogonalization procedure.
\citet{achlioptas2003database} showed the Gaussian distribution can be relaxed to a much simpler distribution only by drawing random binary coins, i.e. $ \frac{1}{\sqrt{m}} \cdot \mathcal{U}( \{1, -1\})^{\otimes (m \times n)}$.
\citet{matouvsek2008variants} generalizes such analytical techniques to i.i.d sub-Gaussian entries $\operatorname{SG}^{\otimes (m \times n)}$.
To further speedup the projection on high-dimensional sparse data, a series of works on design and analyze sparse JL transform \citep{kane2014sparser,cohen2018simple,høgsgaard2023sparse} was proposed. In sparse JL, the column vector could be expressed as entrywise multiplication $\sqrt{\frac{m}{s}}\sigma \odot \eta$ by $\sigma \sim \frac{1}{\sqrt{m}}\mathcal{U}(\{1, -1\})^{\otimes m}$ and a random vector $\eta$ with only $s$ non-zero entries.
These works extends the class of JL distributions.

One alternative is the spherical construction where each column of $\Pi$ is independently sampled from uniform distribution over the sphere $\mathbb{S}^{m-1}$, i.e., $\Pi \sim \mathcal{U}(\mathbb{S}^{m-1})^{\otimes n}$.
Spherical construction was recently shown its superior performance in the application of incremental uncertainty estimation and reinforcement learning via hypermodel~\citep{li2022hyperdqn,li2024hyperagent,dwaracherla2020hypermodels} and epistemic neural networks (ENN)~\citep{osband2023epistemic,osband2023approximate}.
However, existing analysis of JL requires some notion of independence across the entries of each column vector in the random projection matrix $\Pi$ while the spherical construction violates. This limitation comes from the requirement on the sum of independent random variables to facilitates the concentration analysis within the existing probabilistic analytical frameworks.
\begin{blockquote-grey}
    \textbf{Challenge}: Prove that spherical construction is a JL distribution satisfying \Cref{lem:djl}. 
\end{blockquote-grey}

 \begin{table}[!htbp]
    \resizebox{\textwidth}{!}{%
      \begin{tabular}{|l|c|c|c|c|c|c|}
        \hline
        JL dist. (w/o scaling) %
        & $N(0,1)^{\otimes (m \times n)}$ & $\mathcal{U}(\{1, -1\})^{\otimes (m \times n)}$ & $\operatorname{SG}^{\otimes (m \times n)}$ & SJLT & $\mathcal{U}(\mathbb{S}^{m-1})^{\otimes n}$ & $\operatorname{SGV}^{\otimes n}$ \\ \hline
        \citep{indyk1998approximate} %
        & \checkmark   &                                  &                &            &            &                                         \\ \hline
        \citep{achlioptas2003database} %
        &              & \checkmark                       &                &            &          &                                           \\ \hline
        \citep{matouvsek2008variants} %
        & \checkmark   & \checkmark                       & \checkmark     &            &             &                                        \\ \hline
        \citep{kane2014sparser} %
        &              &                                  &                & \checkmark &          &                                           \\ \hline
        \citep{cohen2018simple} %
                &              &                                  &                & \checkmark &  &                                                   \\ \hline
        \citep{høgsgaard2023sparse} %
        &              &                                  &                & \checkmark &   &                                                   \\ \hline

        {\textbf{Our work}} & \checkmark   & \checkmark                       & \checkmark     & \checkmark & \checkmark & \checkmark                                  \\ \hline
      \end{tabular}%
    }
    \caption{What types of constructions can be covered in the literature? SG stands for the distribution of sub-Gaussian random variables in $\R$. SGV stands for the distribution of sub-Gaussian random vectors in $\R^m$. SJLT stands for sparse JL transform introduced in \citep{kane2014sparser}.}
    \label{tab:unified-analysis}
  \end{table}
We provide novel probability tools to resolve this challenge, as one of the contributions highlighted below:
\begin{itemize}[leftmargin=*]
    \item \emph{Analysis of JL}: In \Cref{sec:random-projection}, we present a unified but simple analysis of the Johnson-Lindenstrauss, encompassing spherical, binary-coin, Sparse JL (\Cref{prop:djl-sparse-jl}), Gaussian (\Cref{prop:djl-gaussian}) and sub-Gaussian constructions as particular instances. \Cref{prop:djl-unit-norm} marks the first rigorous demonstration of the spherical construction's efficacy, to the best of our knowledge. Also, with our analytical framework, we discover a new class of sub-Gaussian constructions in \Cref{def:bernstein}, exhibiting potential useful properties. Summaries are in \Cref{tab:unified-analysis}.
    \item \emph{Technical innovations}: Our unified approach to JL analysis leverages an extension of the Hanson-Wright inequality to high dimensions, as detailed in \Cref{thm:hdhw}.
    This tool is essential as it removes the requirement on independence across entries within a column vector of the projection matrix, the key to handle the spherical construction and a more general class of sub-Gaussian constructions.
    While the closest reference we identified is Exercise 6.2.7 in \citep{vershynin2018high}, our extensive review found no existing proofs of this assertion, nor does the mentioned exercise specify concrete constants, unlike our \Cref{thm:hdhw}. Thus, our work in extending the Hanson-Wright inequality to high-dimension, complete with specific proof techniques, represents a significant advancement. Innovations include a novel approach to diagonalization step for the quadratic form.
    \item \emph{Applications}: Leveraging our unified JL analysis and a covering argument, in \Cref{prop:statistics-computation-tradeoff}, we establish a sufficient condition for reduced dimensionality within the context of covariance factorization procedures. This is inspired by the domains of uncertainty estimation and reinforcement learning. 
    Recent neural network models, such as hypermodels~\citep{dwaracherla2020hypermodels,li2022hyperdqn,li2024hyperagent} and epistemic neural networks~\citep{osband2023epistemic,osband2023approximate}, leverage spherical random vectors to update a factorization matrix for incremental uncertainty estimation but lack rigorous guarantees. Our analysis justifies their effectiveness for the first time under the linear setups.
\end{itemize}

\paragraph{Notations.}
We say a random variable $X$ is $K$-sub-Gaussian if  $\E{\exp( \lambda X ) } \le \exp\left( {\lambda^2 K^2}/{2}  \right)$ for all $\lambda \in \R.$
For random variables $X$ in high-dimension $\R^m$, we say it is $K$-sub-Gaussian if for every fixed $v \in \mathbb{S}^{m-1}$ if the scalarized random variable $\langle v, X \rangle$ is $K$-sub-Gaussian.

\section{Simple and unified analysis of Johnson-Lindenstrauss}
\label{sec:random-projection}
In this section, we are going to provide a simple and unified analysis for the following Johnson-Lindenstrauss constructions of random projection matrix satisfying \cref{lem:djl}.
\begin{definition}[Gaussian construction]
\label{def:gaussian-jl}
    Gaussian construction of the random projection matrix $\Pi = ( \rvz_1, \ldots, \rvz_n)$ correspond to each $\rvz_i \sim \frac{1}{\sqrt{m}} N(0, I_m)$ independently.
\end{definition}
\begin{definition}[Binary-coin construction]
    \label{def:binary-jl}
    Binary-coin construction of the random projection matrix $\Pi = ( \rvz_1, \ldots, \rvz_n)$ correspond to each $\rvz_i \sim {\frac{1}{\sqrt{m}}} \mathcal{U}(\{1, -1\}^m)$ independently.
\end{definition}

\begin{definition}[$s$-sparse JL]
\label{def:sparse-jl}
    Sparse JL transform matrix $\Pi = ( \sqrt{\frac{m}{s}} \eta_1 \odot \rvz_1, \ldots, \sqrt{\frac{m}{s}} \eta_n \odot \rvz_n)$ is a random matrix with each $\rvz_i \sim P_{\rvz}$ independently where $P_{\rvz}:= \frac{1}{\sqrt{m}}\mathcal{U}( \{1, -1\}^m )$ and each $\eta_i$ is independently and uniformly sampled from all possible $s$-hot vectors, where $s$-hot vectors is with exactly $s$ non-zero entries with number $1$. This construction is introduced by \citep{kane2014sparser}, also called countsketch.
\end{definition}
Notably, the entries $(\rvz_{i1}, \rvz_{i2}, \ldots, \rvz_{im})$ within the random vector $\rvz_i$ in (1) Gaussian, (2) Binary-coin and (3) sparse JL constructions are mutually independent.
However, the condition on the entry-independence is not true the next construction presented, which brings the major analytical difficulties that have not been discussed in the literature.
\begin{definition}[Spherical construction]
\label{def:sphere-jl}
    Spherical construction of the random projection matrix $\Pi = ( \rvz_1, \ldots, \rvz_n)$ corresponds to each $\rvz_i \sim \mathcal{U}(\mathbb{S}^{m-1})$ independently.
\end{definition}
Before stating our main result for Johnson-Lindenstrauss, we introduce the underlying new probability tool that enables the analysis of spherical construction.
\begin{blockquote-grey}
\begin{theorem}[High-dimensional Hanson-Wright inequality]
  \label{thm:hdhw}
  Let $X_1, \ldots, X_n$ be independent, mean zero random vectors in $\R^m$, each $X_i$ is $K_i$-subGaussian. Let $K = \max_{i} K_i$.
  Let $A = (a_{ij})$ be an $n \times n$ matrix.
  For any $t \ge 0$, we have
  \[
    \prob\left({ \abs{ \sum_{i, j: i\neq j}^n a_{ij} \langle X_i, X_j \rangle} \ge t }\right) \le 2 \exp \left( - \min \left\{ \frac{t^2}{64 m K^4 \norm{A}^2_{F}}, \frac{t}{8 K^2 \norm{A}_{2}} \right\} \right).
  \]
\end{theorem}
\end{blockquote-grey}
\begin{remark}
This is an high-dimension extension of famous Hanson-Wright inequality~\citep{hanson1971bound,wright1973bound,rudelson2013hanson}. The \Cref{thm:hdhw} with exact constant is new in the literature, which maybe of independent interest. Our proof technique generalizes from \citep{rudelson2013hanson} with new treatments on the diagnolization. The proof of \Cref{thm:hdhw} can be found in \Cref{sec:hdhw}. An extension of \Cref{thm:hdhw} to $\sum_{i,j=1}^n a_{ij} \langle X_i, X_j \rangle$ with non-negative diagonal is in \Cref{thm:hdhw-diag}.
\end{remark}
Now, we are ready to provide the unified analysis on Johnson-Lindenstrauss, a simple and direct application of \Cref{thm:hdhw}.
\begin{blockquote-orange}
\begin{proposition}[Binary-coin; Spherical]
\label{prop:djl-unit-norm}
  The Binary-coin and Spherical construction of the random projection matrix $\Pi \in \R^{m \times n}$ in \cref{def:binary-jl,def:sphere-jl} with $m \ge 64 \varepsilon^{-2} \log(2/\delta)$ satisfy \Cref{lem:djl}.
\end{proposition}
\end{blockquote-orange}
\begin{proof}
  From \cref{ex:sphere,ex:unifcube} as will be discussed in \Cref{sec:typical-dist}, we know that the random variables sampled from $\mathcal{U}(\mathbb{S}^{m-1})$ or $\frac{1}{\sqrt{m}}\mathcal{U}(\{1, -1\}^{m})$ are $\frac{1}{\sqrt{m}}$-sub-Gaussian with mean-zero and unit-norm.
  Let $x \in \mathbb{R}^{d}$ be the vector to be projected.
  By the construction of $\Pi$,
  \begin{align}
  \label{eq:decomposition}
      \norm{\Pi x}^2 - \norm{x}^2 = \underbrace{\sum_{1 \le i \neq j \le n} x_i x_j \langle \rvz_i, \rvz_j \rangle}_{\text{off-diagonal}} + \underbrace{\sum_{i=1}^n x_i^2 (\norm{\rvz_i}^2 - 1)}_{\text{diagonal}}
  \end{align}
  As by the condition on unit norm, the diagonal term is zero.
  We apply \Cref{thm:hdhw} with $A = x x^\top$ and $t = \varepsilon \norm{x}^2$. Since $K = 1/ \sqrt{m}$ and $\norm{A}_F = \sqrt{\tr(x x^\top x x^\top )} = \norm{x}^2, \norm{A}_{2} =  \norm{x}^2$, then
  \begin{align*}
    \prob\left({ \abs{ \sum_{1 \le i \neq j \le n} x_i x_j \langle \rvz_i, \rvz_j \rangle} \ge \varepsilon \norm{x}^2 }\right)
     & \le 2 \exp \left( - \min \left\{ \frac{ \varepsilon^2 \norm{x}^4}{ 64 K^4 m \norm{A}^2_{F}}, \frac{ \varepsilon \norm{x}^2}{ 8 \sqrt{2} K^2 \norm{A}_{2}} \right\} \right) \\
     & \le 2 \exp \left( - m \min \left\{ { \varepsilon^2 /64},  \varepsilon/8\sqrt{2} \right\} \right).
  \end{align*}
  This implies that to get the RHS upper bound by $\delta$, we need
  \(
    m \ge 64 \varepsilon^{-2} \log (2/ \delta)
  \)
  for any fixed $\varepsilon \in (0, 1)$.
\end{proof}
\begin{remark}
This proposition is a unified analysis for (1) Spherical construction from random vectors in \cref{ex:sphere} (2) Binary coin construction from random vectors in \cref{ex:unifcube}.
For classical Gaussian construction where $\rvz_i \sim N(0, (1/m)I_m)$ which does not satisfy unit-norm assumption, the diagonal term in \cref{eq:decomposition} is non-zero and needs additional treatments. As analyzed latter in \Cref{prop:djl-gaussian} within the same framework, the requirement for dimension $m = 8 (1+ 2 \sqrt{2})^2 \varepsilon^{-2} \log(2/\delta)$ in the Gaussian construction is larger than the one for Spherical construction. This observation may explain the practical superiority of Spherical construction.
\end{remark}
\begin{remark}
    Reduction of JL to the classical Hanson-Wright~\citep{hanson1971bound,wright1973bound,rudelson2013hanson} has been exploited in \citep{kane2014sparser,cohen2018simple,nelson2020sketching}, e.g. section 5.1 in \citep{nelson2020sketching}. However, as mentioned in \cref{sec:intro}, their analytical assumption on the entry-wise independence, required by the reduction to classical Hanson-Wright, is violated in the spherical construction. Therefore, our high-dimensional extension of Hanson-Wright is crucial for the new unified analysis of JL, accommodating the spherical construction.
\end{remark}

\subsection{Sparse JL transform}
We also present a generalization of \cref{thm:hdhw} that will be helpful to analyze sparse JL transform.
\begin{blockquote-grey}
\begin{theorem}[Generalized High-dimensional Hanson-Wright]
  \label{thm:hdhw-gen}
  Let $b_1, \ldots, b_n$ be fixed vectors in $\R^m$ where $b_{ik}$ is the $k$-th entry of the vector $b_i$.
  Let $X_1, \ldots, X_n$ be independent, mean zero random vectors in $\R^m$, each $X_i$ is $K_i$-subGaussian. Let $K = \max_{i} K_i$. 
  Let $A = (a_{ij})$ be an $n \times n$ matrix.
  For any $t \ge 0$, we have
  \[
    \prob\left({ \abs{ \sum_{i, j: i\neq j}^n a_{ij} \langle b_i \odot X_i, b_j \odot X_j \rangle} \ge t }\right) \le 2 e^{ - \min \left\{ \frac{t^2}{64 K^4 \sum_{k=1}^m \norm{A^b_k}^2_{F}}, \frac{t}{8 K^2 \max_{k} \norm{A^b_k}_{2}} \right\}}.
  \]
  where ${A}^b_k$ is a matrix with entries ${A}^b_k(i,j) = a_{ij} b_{ik} b_{jk} $ for each $(k, i, j) \in [m]\times [n] \times [n]$.
\end{theorem}
\end{blockquote-grey}
\Cref{thm:hdhw-gen} extends \Cref{thm:hdhw} in a way that each random vector $X_i$ is entry-wise scaled by corresponding $b_i$ for $i \in [n]$. When $b_1 = b_2 =\cdots = b_n = \mathbf{1}$ is all-one vector, it reduces to \Cref{thm:hdhw}. The proof of \Cref{thm:hdhw-gen} is similar to \Cref{thm:hdhw} and is deferred to \Cref{sec:general-hdhw}. 
Now we are ready to include the sparse JL construction into our unified analytical framework.
\begin{blockquote-orange}
\begin{proposition}
\label{prop:djl-sparse-jl}
  The sparse JL construction in \cref{def:sparse-jl} with $m \simeq \varepsilon^{-2} \log(1/\delta)$ and $s \simeq \varepsilon^{-1} \log(1/\delta)$ satisfies \Cref{lem:djl}.
\end{proposition}
\end{blockquote-orange}
\begin{proof}
  From \cref{ex:unifcube}, we know that $\rvz_i \sim P_{\rvz} = \frac{1}{\sqrt{m}}\mathcal{U}(\{1, -1\}^m)$ is a $\frac{1}{\sqrt{m}}$-sub-Gaussian random vector with mean zero and unit-norm.
  Let $x \in \mathbb{R}^{d}$ be the vector to be projected.
  By the construction of $\Pi$,
  \begin{align}
  \label{eq:decomposition-sparse}
      \norm{\Pi x}^2 - \norm{x}^2 = \underbrace{\sum_{1 \le i \neq j \le n} \frac{m}{s} x_i x_j \langle \eta_i \odot \rvz_i, \eta_j \odot \rvz_j \rangle}_{\text{off-diagonal}} + \underbrace{\sum_{i=1}^n x_i^2 ( \frac{m}{s}\norm{\eta_i \odot \rvz_i}^2 - 1)}_{\text{diagonal}}
  \end{align}
  By the sparse JL construction in \cref{def:sparse-jl}, the diagonal term in \cref{eq:decomposition-sparse} is zero. 
  W.L.O.G, we assume that $\norm{x}^2 = 1$.
  We could apply \Cref{thm:hdhw-gen} conditioned on $(\eta_i)_i$ with $A = (m/s)x x^\top$, $(b_i) = (\eta_i)$ and $t = \varepsilon$.
  The constructed matrix in the \Cref{thm:hdhw-gen} will be $A^b_k = \frac{m}{s} (x \odot \eta^k) (x \odot \eta^{k})^\top $ where $\eta^k = ( \eta_{1k}, \eta_{2k}, \ldots, \eta_{nk})$.
  Indeed, $\norm{A^b_k}_F = \sum_{ij} (m/s)^2 x_i^2 x_j^2 \eta_{ik} \eta_{jk}$ and $\norm{A^b_k}_{2} = (m/s) \norm{ (x \odot \eta^{k}) }^2_2 \le (m/s)$. Since $K = 1/ \sqrt{m}$, \Cref{thm:hdhw-gen} yields,
  \begin{align*}
    \prob\left( \text{off-diagonal} \ge \varepsilon \mid (\eta_i)_{i=1}^n \right) \le 2 \exp\left(- \frac{\varepsilon^2}{ 64 ({1}/{s^2}) \sum_{k=1}^m \sum_{i,j} x_i^2 x_j^2 \eta_{ik} \eta_{jk} } \right) + 2\exp\left( - \frac{\varepsilon}{8\sqrt{2} (1/s) } \right).
  \end{align*}
  With a translation of tail bound to moment bound in \cref{lem:equivalent-condition},
  \begin{align}
  \label{eq:conditional-moment}
        \underbrace{\left(\E{ \abs{ \text{off-diagonal} }^p \mid (\eta_i)_{i=1}^n } \right)^{1/p}}_{(a)} \lesssim \frac{\sqrt{p}}{s} \sqrt{ \sum_{ij} x_i^2 x_j^2 \sum_{k=1}^m \eta_{ik} \eta_{jk} } + \frac{p}{s}.
  \end{align}
  Then by the tower property and \cref{eq:conditional-moment}
  \begin{align}
  \label{eq:expand-moment}
  \left( \E{ \abs{ \text{off-diagonal} }^p} \right)^{1/p}
    & = \left(\E { (a)^p } \right)^{1/p} \nonumber \\
      & \lesssim \left( \mathbb{E}\left( \frac{\sqrt{p}}{s} \sqrt{ \sum_{ij} x_i^2 x_j^2 \sum_{k=1}^m \eta_{ik} \eta_{jk} } + \frac{p}{s}\right)^p \right)^{1/p} \nonumber \\
      & \le \frac{\sqrt{p}}{s} \underbrace{\left( \mathbb{E}\left( \sqrt{ \sum_{ij} x_i^2 x_j^2 \sum_{k=1}^m \eta_{ik} \eta_{jk} } \right)^p \right)^{1/p}}_{(b)} + \frac{p}{s},
  \end{align}
  where the last inequality is by triangular inequality of $L_p$-norm.
  The term $(b)$ can be bounded as follows when $p \simeq s^2/m$,
  \begin{align}
  \label{eq:upper-bound-moment}
      (b) \stackrel{(1)}{\le} \sqrt{ \sum_{ij} x_i^2 x_j^2  \left( \mathbb{E}\left( \sum_{k=1}^m \eta_{ik} \eta_{jk}  \right)^p \right)^{1/p} } \stackrel{(2)}{\lesssim} \sqrt{ \sum_{ij} x_i^2 x_j^2  p } = \sqrt{p},
  \end{align}
  where $(1)$ is by Jensen's inequality; $(2)$ follows by \cref{lem:bernoulli-complex-moment} as $\left( \mathbb{E}\left( \sum_{k=1}^m \eta_{ik} \eta_{jk}  \right)^p \right)^{1/p} \lesssim \sqrt{s^2/m} \cdot \sqrt{p} + p \simeq p$ when $p\simeq s^2/m$; and the last equality follows the assumption $\norm{x}^2 = 1$, resulting $\sum_{i}x_i^2 \sum_{j} x_j^2 = 1 \cdot 1$.
  Therefore, plugging the upper bound in \cref{eq:upper-bound-moment} with $p \simeq s^2/m$ into \cref{eq:expand-moment},
  \[
  \left( \E{ \abs{ \text{off-diagonal} }^p} \right)^{1/p} \lesssim \sqrt{\frac{p}{m}} + \frac{p}{s} \simeq \frac{p}{s} \simeq \frac{s}{m}
  \]
  Then by Markov's inequality and the settings of $p\simeq s^2/m, s\simeq \varepsilon m, m \simeq \varepsilon^{-2} \log(1/\delta)$, 
  \begin{align*}
      \mathbb{P}\left(\left|\|\Pi x\|_2^2-1\right|>\varepsilon\right)
    & =\mathbb{P}\left( \abs{\text{off-diagonal}} >\varepsilon\right) < \varepsilon^{-p} \cdot \E{\abs{\text{off-diagonal}}^p } \\
    & <\varepsilon^{-p} \cdot (\frac{s}{m})^p \cdot C^p < C^{\log(1/\delta)} < \delta,
  \end{align*}
  where $C$ is some constant as a result of configuration in $p, m, s$ for the purpose.
\end{proof}

\begin{lemma}
\label{lem:bernoulli-complex-moment}
For $\eta_i, i= 1, \ldots, n$ defined in \cref{def:sparse-jl}, the $p$-th moment of $\sum_{k=1}^m \eta_{ik} \eta_{jk}$ is bounded 
\[
\left( \mathbb{E}\left( \sum_{k=1}^m \eta_{ik} \eta_{jk}  \right)^p \right)^{1/p} \lesssim \sqrt{s^2 / m} \cdot \sqrt{p}+p
\]
\end{lemma}
\begin{proof}
    Suppose the event $I$ is that $\eta_{i,a_1}, \ldots, \eta_{i,a_s}$ are all 1 , where $a_1<a_2<\ldots<a_s$.
    Note that conditioned on event $I$, the sum $\sum_{k=1}^m \eta_{ik} \eta_{jk}$ can be written as $\sum_{k=1}^s Y_k$, where $Y_k$ is an indicator random variable for the event that $\eta_{j, a_k}=1$. The $(Y_k)_{k= 1}^s$ are not independent, but for any integer $p \geq 1$ their $p$ th moment is upper bounded by the case that the $(Y_k)_{k=1}^s$ are independent Bernoulli each of expectation $(s / m)$ (this can be seen by simply expanding $\left(\sum_{k=1}^s Y_k\right)^p$ then comparing with the independent Bernoulli case monomial by monomial in the expansion as shown in \citep{cohen2018simple}). Finally, via the moment version of the Bernstein inequality, we obtain
    \[
        \left( \mathbb{E}\left( \sum_{k=1}^s Y_k \right)^p \right)^{1/p} \lesssim \sqrt{s \frac{s}{m} \left(1- \frac{s}{m}\right)} \cdot \sqrt{p}+p \le \sqrt{ \frac{s^2}{m} } \cdot \sqrt{p}+p .
    \]
    The lemma follows from taking the expectation over the event $I$ and the tower property of expectation,
    \[
        \E{ \left(\sum_{k=1}^m \eta_{ik} \eta_{jk}\right)^p} = \E{\E{ \left(\sum_{k=1}^m \eta_{ik} \eta_{jk}\right)^p \bigg| I}} = \E{\E{ \left(\sum_{k=1}^s Y_k \right)^p \bigg| I}}.
    \]
\end{proof}

\begin{lemma}[Theorem 2.3 in \citep{boucheron2003concentration}]
\label{lem:equivalent-condition}
Let $Z$ be a scalar random variable. The following statements are equivalent.
(a) There exist $\sigma, K>0$ s.t. $\forall p \geq 1,\|Z\|_p \leq C_a(\sigma \sqrt{p}+K p)$.
(b) There exist $\sigma, K>0$ s.t. $\forall \lambda>0, \mathbb{P}(|Z|>\lambda) \leq C_b\left(e^{-C_b^{\prime} \lambda^2 / \sigma^2}+e^{-C_b^{\prime} \lambda / K}\right)$.
The constants $C_a$, $C_b$ and $C'_b$ change by at most some absolute constant factor.
\end{lemma}

\subsection{General sub-Gaussian construction without unit-norm}
\label{sec:no-unit-norm}
In this section, we consider the cases where the diagonal term in the decomposition (\cref{eq:decomposition}) is non-zero. We need additional conditions to guarantee \Cref{lem:djl}, a two-sided probability bound.
Before diving into the general treatment of sub-Gaussian setups, let us first look at the classical Gaussian construction in \cref{def:gaussian-jl} where the column vector does not satisfy the unit-norm condition and we could get some intuition on more general case.
\begin{blockquote-orange}
\begin{proposition}[Gaussian]
\label{prop:djl-gaussian}
The Gaussian construction of the random projection matrix $\Pi \in \R^{m \times n}$ in \cref{def:gaussian-jl} with $m \ge 8 (1+ 2 \sqrt{2})^2 \varepsilon^{-2} \log(2/\delta)$ satisfy \Cref{lem:djl}.
\end{proposition}
\end{blockquote-orange}
\begin{remark}
    The required dimension $m = 8 (1+ 2 \sqrt{2})^2 \varepsilon^{-2} \log(2/\delta)$ in the Gaussian construction to guarantee \cref{lem:djl} is larger than the one $m = 64 \varepsilon^{-2} \log(2/\delta)$ in spherical and binary coin construction as shown in \cref{prop:djl-unit-norm}. Since we analyze these constructions within the same analytical framework, the smaller $m$ in Spherical construction may explain its practical superiority.
\end{remark}
\begin{proof}
    The random variables sampled from $N(0, \frac{1}{m}I_m)$ are $\frac{1}{\sqrt{m}}$-sub-Gaussian with mean-zero.
    The off-diagonal term as decomposed in \cref{eq:decomposition} can be dealt as the same in \cref{prop:djl-unit-norm} via \cref{thm:hdhw}. However, the diagonal term is non-zero in Gaussian construction. Notice that, the diagonal term 
    \(
        \sum_{i=1}^n x_i^2 (\norm{\rvz_i}^2 - 1),
    \)
    is essentially a weighted sum of i.i.d. $\chi^2_m$ random variables. Let $Z_{ij} \sim N(0, 1)$ for all $(i, j) \in [n] \times [m]$.
    \begin{align}
        \E{ \exp( \lambda \sum_{i=1}^n x_i^2 (\norm{\rvz_i}^2 - 1) ) } 
        & = \E{ \exp\left( \sum_{i=1}^n \sum_{j = 1}^m \frac{\lambda x_i^2}{m} (Z_{ij}^2 - 1) \right) }.
    \end{align}
    As $\max_{i} \lambda x_i^2 / m \le 1/2$, the moment generating function of the diagonal terms will become
    \begin{align}
        \prod_{i=1}^n \prod_{j=1}^m \frac{\exp(-\lambda x_i^2 /m )}{ \sqrt{ 1- 2\lambda x_i^2 / m }} \le \exp \left( m \cdot \frac{2\lambda^2}{m^2} \sum_{i} x_i^4 \right),
        \quad \forall \abs{\lambda} < \frac{m}{4 \max_i x_i^2},
    \end{align}
    where the last inequality is due to $\frac{\exp(-x)}{\sqrt{1-2x}} \le \exp 2 x^2$ for $\abs{x} <1/4$.
    Notice $\max_{i} x_i^2 = \norm{x}_{\infty}^2$. Finally, we have,
    \begin{align*}
        \prob\left( \sum_{i=1}^n x_i^2 (\norm{\rvz_i}^2 - 1) \ge t \norm{x}^2 \right)
        & \le \inf_{\abs{\lambda} < \frac{m}{4\norm{x}_{\infty}^2} }\exp( - \lambda t + 2\lambda^2 \sum_{i} x_i^4 / m) \\
        & = \exp\left( - m \cdot \min\left\{ \frac{t^2}{8\sum_{i}x_i^4}, \frac{t}{8 \norm{x}_{\infty}^2} \right\} \right).
    \end{align*}
    As we need to deal with diagonal term separately with the off-diagonal term in \cref{eq:decomposition}, say let $\varepsilon = \varepsilon_1 + \varepsilon_2$,
    \begin{align*}
        \prob( \abs{\norm{\Pi x}^2 - \norm{x}^2} \ge \varepsilon \norm{x}^2 ) 
        & \le \prob( \abs{ \text{off-diagonal} } \ge \varepsilon_1\norm{x}^2 ) + \prob( \abs{ \text{diagonal} } \ge \varepsilon_2\norm{x}^2) \\
        & \le 2 \exp \left( - m \cdot \min \left\{ { \frac{\varepsilon_1^2}{64}},  \frac{\varepsilon_1}{8\sqrt{2}} \right\} \right) + 2\exp\left( - m \cdot \min\left\{ \frac{\varepsilon_2^2}{8}, \frac{\varepsilon_2}{8} \right\} \right),
    \end{align*}
    where the last inequality is true due to the fact $\norm{x}_{\infty}^2 \le \norm{x}^2$ and $\sum_{i} x_i^4 < \norm{x}^4$.
    Let $\varepsilon_1 = \frac{2\sqrt{2}}{1+2\sqrt{2}} \varepsilon$ and $\varepsilon_2 = \frac{1}{1+2 \sqrt{2}} \varepsilon$, we conclude in the Gaussian construction of $\Pi$
    \begin{align*}
        \prob( \abs{\norm{\Pi x}^2 - \norm{x}^2} \ge \varepsilon \norm{x}^2) 
        & \le 4 \exp\left( - \frac{m \varepsilon^2}{ 8 (1+ 2\sqrt{2})^2 } \right).
    \end{align*}
    To guarantee \Cref{lem:djl}, we require $m \ge 8 (1+ 2 \sqrt{2})^2 \varepsilon^{-2} \log ( 4 /\delta) $.
\end{proof}
In general, we cannot expect a lower tail bound for the squared norm of sub-Gaussian random variables in high dimension. Since \cref{lem:djl} is a two-sided tail bound, we make the following Bernstein-type assumption on the squared norm, in addition to the mean-zero independent sub-Gaussian condition.
\begin{blockquote-orange}
\begin{definition}[Sub-Gaussian construction with Bernstein condition]
\label{def:bernstein}
Sub-Gaussian construction of the random projection matrix $\Pi = (\rvz_1, \ldots, \rvz_n)$ has each column $\rvz_i$ being independent $\sqrt{1/m}$-sub-Gaussian random variable in $\R^m$ with mean zero. Additionally, there exists a universal constant $C>0$ such that
\[
\mathbb{E}\left|\left\|\rvz_i\right\|^2-\mathbb{E}\left\|\rvz_i\right\|^2\right|^k \leq C k ! \left(\frac{1}{m}\right)^{\frac{k-2}{2}} \quad \forall k=3,4, \ldots
\]
\end{definition}
\end{blockquote-orange}
\begin{remark}
    Gaussian construction in \cref{def:gaussian-jl} is a special case of the sub-Gaussian construction in \cref{def:bernstein} as $\chi^2_m$ satisfies the Bernstein condition. Meanwhile, the sub-Gaussian construction in \cref{def:bernstein} generalize the spherical and binary-coin constructions. As we do not assume the random vector in each column has fixed norm, this also relax the analytical assumption of the Theorem 5.58 in \citep{vershynin2012introduction} for extreme singular value of random matrix with independent sub-Gaussian columns.
\end{remark}
\begin{remark}
    Sub-Gaussian construction in \cref{def:bernstein} requires the same order of $m$ as in Gaussian construction to guarantee \cref{lem:djl}. The proof is a direct application of the Composition property of sub-Exponential random variables \citep{vershynin2018high,wainwright2019high}.
\end{remark}

\section{Typical sub-Gaussian distributions}
\label{sec:typical-dist}
In this section, we examine the properties of typical distribution for construction random projection matrix. Specifically, we examine sub-Gaussian condition of two high-dimensional distributions: (1) Uniform distribution over the unit sphere, and (2) Uniform distribution over the scaled cube.
Before diving to the details, we first introduce a useful lemma on centered MGF for Beta distribution with a tight sub-Gaussian constant.
\begin{lemma}[MGF of Beta distribution]
  \label{lem:mgf-beta}
  For any $\alpha, \beta \in \R_+$ with $\alpha \ge \beta$.
  Random variable $X \sim \operatorname{Beta}(\alpha, \beta)$ has variance $\var{X} = \frac{\alpha \beta}{(\alpha + \beta)^2 (\alpha + \beta + 1)}$ and the centered MGF 
  $$\E{\exp(\lambda(X - \E{X}))} \le \exp\left( \frac{\lambda^2 \var{X}}{2} \right).$$
\end{lemma}
\begin{remark}
    The constant in \cref{lem:mgf-beta} is new in the literature and seems to be tight as it already achieve the same constant in the MGF of Gaussian distribution with variance $\var{X}$.
\end{remark}
\begin{proof}
  For $X \sim \operatorname{Beta}(\alpha, \beta)$, \citet{skorski2023bernstein} gives a novel order-2-recurrence for central moments.
  \begin{align*}
    \mathbb{E}\left[(X-\mathbb{E}[X])^p\right]= & \frac{(p-1)(\beta-\alpha)}{(\alpha+\beta)(\alpha+\beta+p-1)} \cdot \mathbb{E}\left[(X-\mathbb{E}[X])^{p-1}\right] \\ & +\frac{(p-1) \alpha \beta}{(\alpha+\beta)^2(\alpha+\beta+p-1)} \cdot \mathbb{E}\left[(X-\mathbb{E}[X])^{p-2}\right]
  \end{align*}
  Let $m_p := \frac{\mathbb{E}\left[(X-\mathbb{E}[X])^p\right]}{p!}$,
  When $\alpha \ge \beta$, it follows that $m_p$ is non-negative when $p$ is even, and negative otherwise.
  Thus, for even $p$,
  $$m_p \le \frac{1}{p} \cdot \frac{\alpha \beta }{(\alpha + \beta)^2(\alpha+\beta+p-1)} m_{p-2} \le \frac{\var{X}}{p} \cdot m_{p-2}.$$
  Repeating this $p / 2$ times and combining with $m_p \leqslant 0$ for odd $p$, we obtain
  $$m_p \leqslant \begin{cases}\frac{\var{X}^{\frac{p}{2}}}{p ! !} & p \text { even } \\ 0 & p \text { odd }\end{cases}.$$
  Using $p ! !=2^{p / 2}(p / 2) !$  for even $p$, for $t \geqslant 0$ we obtain
  \[
    \E{\exp(\lambda [X - \E{X}])} \leqslant 1+\sum_{p=2}^{+\infty} m_p \lambda^p = 1 + \sum_{p=1}^{+\infty} (\lambda^2 \var{X}/2)^p/p! = \exp \left(\frac{\lambda^2 \var{X}}{2}\right)
  \]
\end{proof}
\begin{example}[Uniform distribution over $m$-dimensional sphere $\mathcal{U}(\mathbb{S}^{m-1})$]
  \label{ex:sphere}
    Unit-norm condition is trivial to verify.
  Given a random vector $\rvz \sim \mathcal{U}(\mathbb{S}^{m-1})$, for any $v \in \mathbb{S}^{m-1}$, we have
  \[
    \langle \rvz, v \rangle \sim 2 \operatorname{Beta}\left(\frac{m-1}{2}, \frac{m-1}{2}\right)-1.
  \]
  Thus, by \cref{lem:mgf-beta}, we confirm that the random variable $\rvz \in \R^m$ is $\frac{1}{\sqrt{m}}$-sub-Gaussian.
\end{example}

\begin{example}[Uniform distribution over scaled $m$-dimensional cube]
    \label{ex:unifcube}
    The random variable
    $\rvz \sim \frac{1}{\sqrt{m}} \cdot \mathcal{U}(\{1, -1\}^m)$ is $\frac{1}{m}$-sub-Gaussian and with unit-norm.
    This is because we could sample the random vector $\rvz$ by sample each entry independently from $\rz_{i} \sim \frac{1}{\sqrt{m}}\mathcal{U}(\{1, -1\})$ for $i \in [m]$.
    Then, for any $v \in \mathbb{S}^{m-1}$, by independence,
    \begin{align*}
        \E{\exp(\lambda \langle v, \rvz \rangle)} = \prod_{i=1}^m \E{\exp(\lambda v_i \rz_i)}
        \le 
        \prod_{i=1}^m \exp(\lambda^2 v^2_i / 2m) = \exp(\lambda^2 \sum_{i} v^2_i / 2m).
    \end{align*}
    The inequality is due to MGF of rademacher distribution (e.g. Example 2.3 in \citep{wainwright2019high}).
\end{example}

\section{Proof of High-dimensional Hanson-Wright in \Cref{thm:hdhw}}
\label{sec:hdhw}

\begin{proof}
  We prove the one-side inequality and the other side is similar by replacing $A$ with $-A$.
  Let
  \begin{align}
    \label{eq:def-S}
    S = \sum_{i, j: i \neq j}^n a_{ij} \langle X_i, X_j \rangle.
  \end{align}
  \textbf{Step 1: decoupling.}
  Let $\iota_1, \ldots, \iota_d \in \{0, 1\}$ be symmetric Bernoulli random variables, (i.e., $\prob(\iota_i = 0) = \prob(\iota_i = 1) = 1/2$) that are independent of $X_1, \ldots, X_n$.
  Since
  \[
    \E{\iota_i (1 - \iota_i)} =
    \begin{cases}
      0,   & i = j,    \\
      1/4, & i \neq j,
    \end{cases}
  \]
  we have $S = 4 \E[\iota]{S_{\iota} } $, where
  \[
    S_\iota = \sum_{i, j = 1}^n \iota_i(1-\iota_j) a_{ij} \langle X_i, X_j \rangle
  \]
  and the expectation $\E[\iota]{\cdot}$ is the expectation taken with respect to the random variables $\iota_i$.
  By Jensen's inequality and $\exp(\lambda x)$ is a convex function in $x$ for any $\lambda \in \R$, we have
  \[
    \E{\exp (\lambda S)} \le \E[X, \iota]{ \exp (4 \lambda S_{\iota} )}.
  \]
  Let $\Lambda_{\iota} = \{ i \in [d]: \iota_i = 1  \}$. Then we write
  \[
    S_{\iota} = \sum_{i \in \Lambda_{\iota}} \sum_{ j \in \Lambda_{\iota}^c } a_{ij} \langle X_i, X_j \rangle = \sum_{j \in \Lambda_{\iota}^{c} } \langle \sum_{i \in \Lambda_{\iota} } a_{ij} X_i, X_j \rangle.
  \]
  Taking expectation over $(X_j)_{j \in \Lambda_{\iota}^c } $ (i.e., conditioning on $(\iota_i)_{i= 1, \ldots, d}$ and $(X_i)_{i \in \Lambda_{\iota}}$ ), it follows that
  \[
    \E[(X_j)_{j \in \Lambda_{\iota}^c }]{
      \exp (4 \lambda S_{\iota})
    }
    = \prod_{j \in \Lambda_{\iota}^c} \E[(X_j)_{j \in \Lambda_{\iota}^c }]{
      e^{ 4\lambda \langle \sum_{i \in \Lambda_{\iota} } a_{ij} X_i, X_j \rangle}
    }
  \]
  by the independence among $(X_j)_{j\in \Lambda_{\iota} }$.
  By the assumption that $X_j$ are independent sub-Gaussian with mean zero, we have
  \[
    \E[(X_j)_{j \in \Lambda_{\iota}^c }]{
      \exp (4 \lambda S_{\iota})
    }
    \le \exp \left( \sum_{j \in \Lambda_{\iota}^c} 8 \lambda^2 K_j^2 \norm{ \sum_{i \in \Lambda_{\iota} } a_{ij} X_i }^2 \right) =: \exp \left( 8 \lambda^2 \sigma_{\iota}^2 \right).
  \]
  Thus we get
  \[
    \E[X]{ \exp (4 \lambda S_{\iota}) } \le \E[X]{ \exp (8 \lambda^2 \sigma_{\iota}^2) }.
  \]
  \textbf{Step 2: reduction to Gaussian random variables.}
  For $j=1, \ldots, n$, let $g_j$ be independent $N\left(0,16 K_j^2 \mI \right)$ random variables in $\R^m$ that are independent of $X_1, \ldots, X_n$ and $\iota_1, \ldots, \iota_n$. Define
  \[
    T:=\sum_{j \in \Lambda_\iota^c}\langle g_j, \sum_{i \in \Lambda_\iota} a_{i j} X_i \rangle.
  \]
  Then, by the definition of Gaussian random variables in $\R^m$, we have
  \[
    \begin{aligned}
      \E[g]{\exp{ (\lambda T}) }
       & = \prod_{j \in \Lambda_\iota^c} \E[g]{ e^{\langle g_j, \lambda \sum_{i \in \Lambda_\iota} a_{i j} X_i \rangle} }  \\
       & =\exp \left(8 \lambda^2 \sum_{j \in \Lambda_\iota^c} K_j^2 \norm{ \sum_{i \in \Lambda_\iota} a_{i j} X_i }^2 \right)
      = \exp \left(8 \lambda^2 \sigma_\iota^2\right)
    \end{aligned}
  \]
  So it follows that
  \[
    \E[X]{\exp {(4 \lambda S_\iota})} \leq \E[{X, g}]{ \exp {(\lambda T)} }.
  \]
  Since $T = \sum_{i \in \Lambda_\iota} \langle\sum_{j \in \Lambda_\iota^c} a_{i j} g_j, X_i \rangle$, by the assumption that $X_i$ are independent sub-Gaussian with mean zero, we have
  \[
    \E[\left(X_i\right)_{i \in \Lambda_\iota}]
    {\exp {(\lambda T)}}
    \leq \exp \left(\frac{\lambda^2}{2} \sum_{i \in \Lambda_\iota} K_i^2 \norm{ \sum_{j \in \Lambda_\iota^c} a_{i j} g_j}^2 \right),
  \]
  which implies that
  \begin{align}
    \label{eq:reduce-gaussian}
    \E[X]{\exp {(4 \lambda S_\iota)} } \le \E[g]{ \exp \left(\lambda^2 \tau_\iota^2 / 2\right) }
  \end{align}
  where $\tau_\iota^2=\sum_{i \in \Lambda_\iota} K_i^2 \norm{ \sum_{j \in \Lambda_\iota^c} a_{i j} g_j}^2$. Note that $\tau_\iota^2$ is a random variable that depends on $(\iota_i)_{i=1}^d$ and $(g_j)_{j=1}^n$.

  \textbf{Step 3: diagonalization.}
  We have $ g_j=\sum_{k=1}^{m} \left\langle g_j, e_k\right\rangle e_k$ and
  \begin{align*}
    \tau_\iota^2
     & =\sum_{i \in \Lambda_\iota} K_i^2\left\|\sum_{j \in \Lambda_\iota^c} a_{i j} g_j\right\|^2=\sum_{i \in \Lambda_\iota} K_i^2\left\|\sum_{k=1}^{m} \left( \sum_{j \in \Lambda_\iota^c} a_{i j} \left\langle g_j, e_k \right\rangle\right) e_k\right\|^2 \\
     & = \sum_{k=1}^{m} \sum_{i \in \Lambda_\iota}\left(\sum_{j \in \Lambda_\iota^c} K_i a_{i j}\left\langle g_j, e_k\right\rangle\right)^2                                                                                                                  \\
     & = \sum_{k=1}^m \norm{P_{\iota}\tilde{A} (I - P_\iota) G_k }^2
  \end{align*}
  where the last second step follows from Parseval's identity.
  $G_{j k}:=\left\langle g_j, e_k\right\rangle, j=1, \ldots, n$, are independent $N\left(0,16 K_j^2\right)$ random variables.
  $G_k = (G_{1k}, \ldots, G_{nk} )^\top \in \R^n$.
  $\widetilde{A}=\left(\tilde{a}_{i j}\right)_{i, j=1}^n$ with $\tilde{a}_{i j}=K_i a_{i j}$.
  Let $P_\iota \in \R^{n \times n}$ be the restriction matrix such that $P_{\iota, ii}=1$ if $i \in \Lambda_\iota$ and $P_{\iota, ij}=0$ otherwise.

  Define normal random variables $Z_k = (Z_{1k}, \ldots, Z_{nk})^\top \sim N(0,I)$ for each $k = 1, \ldots, m$.
  Then we have $G_k \stackrel{D}{=} \Gamma^{1/2} Z_k$ where $\Gamma = 16 \diag({ K_1^2}, \ldots, {K_n^2})$.

  Let $\tilde{A}_{\iota} := P_{\iota}\tilde{A}(I - P_{\iota})$.
  Then by the rotational invariance of Gaussian distributions, we have
  \begin{align*}
    \sum_{k=1}^m \norm{\tilde{A}_{\iota} G_k}^2 \stackrel{D}{=} \sum_{k=1}^m \norm{\tilde{A}_{\iota} \Gamma^{1/2} Z_k}^2  \stackrel{D}{=} \sum_{k=1}^m \sum_{j=1}^n s_j^2 Z_{jk}^2
  \end{align*}
  where $s_j^2, j = 1, 2, \ldots, n$ are the eigenvalues of $\Gamma^{1/2} \tilde{A}_{\iota}^\top \tilde{A}_{\iota} \Gamma^{1/2}$.

  \textbf{Step 4: bound the eigenvalues.}
  It follows that
  \[
    \max _{j \in [n]} s_j^2 = \norm{\tilde{A}_{\iota} \Gamma^{1/2} }^2_2 \le 16 K^4 \norm{A}_2^2.
  \]
  In addition, we also have
  \begin{align*}
    \sum_{j=1}^n s_j^2 & = \tr ( \Gamma^{1/2} \tilde{A}_{\iota}^\top \tilde{A}_{\iota} \Gamma^{1/2} ) \le 16 K^4 \|A\|_{F}^2
  \end{align*}
  and $\sum_{k=1}^m \sum_{j=1}^n s_j^2 \le 16 m K^4 \norm{A}_{F}^2$.
  Invoking \cref{eq:reduce-gaussian}, we get
  \[
    \mathbb{E}_X\left[\exp{(4 \lambda S_\iota)}\right]
    \leq \prod_{k=1}^{m} \prod_{j=1}^n \mathbb{E}_Z \left[\exp \left(\lambda^2 s_j^2 Z_{jk}^2 / 2\right)\right]
  \]
  Since $Z_{jk}^2$ are i.i.d. $\chi_1^2$ random variables with the moment generating function $\E{ \exp {(t Z_{jk}^2)} } = (1 - 2 t)^{-1 / 2}$ for $t < 1 / 2$, we have
  \[
    \mathbb{E}_X\left[\exp{(4 \lambda S_\iota)}\right] \leq \prod_{k=1}^{m} \prod_{j=1}^n \frac{1}{\sqrt{1-\lambda^2 s_j^2}} \quad \text { if } \max_j \lambda^2 s_j^2 < 1.
  \]
  Using $(1-z)^{-1 / 2} \leq \exp (z)$ for $z \in[0,1 / 2]$, we get that if $\lambda^2 \max_{j} s_j^2 \le 1/2$, i.e., $32 K^4 \norm{A}_2^2 \lambda^2<1$, then
  \[
    \mathbb{E}_X\left[\exp{(4 \lambda S_\iota)}\right] \leq \exp \left(\lambda^2 \sum_{k=1}^{m} \sum_{j = 1}^n s_j^2\right) \leq \exp \left(16 \lambda^2 m K^4 \|A\|_{F}^2\right) .
  \]
  Note that the last inequality is uniform in $\iota$. Taking expectation with respect to $\delta$, we obtain that
  \[
    \mathbb{E}_X\left[\exp{(\lambda S)}\right] \leq \mathbb{E}_{X, \iota}\left[\exp{(4 \lambda S_\iota)}\right] \leq \exp \left(16 \lambda^2 m K^4 \|A\|_{F}^2\right)
  \]
  whenever $\abs{\lambda}<(4 \sqrt{2} K^2\norm{A}_2 )^{-1}$.

  \textbf{Step 5: Conclusion.} Now we have
  \[
    \mathbb{P}(S \geq t) \leq \exp \left(-\lambda t+16 \lambda^2 m K^4 \|A\|_{F}^2\right) \quad \text { for } \abs{\lambda} \leq\left(4 \sqrt{2} K^2\norm{A}_2 \right)^{-1}
  \]
  Optimizing in $\lambda$, we deduce that there exists a universal constant $C>0$ such that
  \[
    \mathbb{P}(S \geq t) \leq \exp \left[- \min \left(\frac{t^2}{64 m K^4 \|A\|_{F}^2}, \frac{t}{ 8 \sqrt{2} K^2 \norm{A}_2}\right)\right].
  \]
\end{proof}

\section{Application in Uncertainty Estimation}
\label{sec:streaming}
Folklore suggests scalable and incremental uncertainty estimation through hypermodels~\citep{dwaracherla2020hypermodels,li2022hyperdqn,li2024hyperagent} and epistemic neural networks (ENN)~\citep{osband2023epistemic,osband2023approximate}, yet no rigorous guarantees exist. These works consider settings where feature vectors $x_t \in \mathbb{R}^d$ for $t = 1, \ldots, T$ appear in a streaming fashion. This data stream assumption is grounded in reinforcement learning, where an agent interacts with environments and receives new observations sequentially. 

\citet{li2022hyperdqn} summarize the closed-form incremental algorithm in linear setups, where it incrementally updates an $\mathbb{R}^{d \times M}$ matrix $\rmA$ using the sequences $(x_t)_{t \ge 1}$ and $(\rvz_t)_{t \ge 1}$, resulting in a matrix at time $T$ given by
\begin{align}
\label{eq:factor}
    \rmA = \mSigma \left(\mSigma_0^{-1/2} \rmZ_0 + \frac{1}{\sigma} \sum_{t=1}^T x_t \rvz_t^\top \right),
\end{align}
where (1) $\rmZ_{0} \in \mathbb{R}^{d \times M}$ and $\rvz_t \in \mathbb{R}^M$ are algorithm-generated random matrix and random vectors, and (2) $\mSigma = \left(\mSigma_0^{-1} + \frac{1}{\sigma^2}\sum_{t=1}^T x_t x_t^\top \right)$ is the posterior covariance matrix. Here, $\mSigma_0 \in \mathbb{R}^{d \times d}$ is the prior covariance matrix and $\sigma$ is the standard deviation of the response noise in the linear-Gaussian model.

\citet{li2022hyperdqn,dwaracherla2020hypermodels,osband2023epistemic} typically generate these random vectors using spherical distribution and state that the goal is to ensure the matrix $\rmA$ is an approximate factorization of the posterior covariance matrix $\mSigma$, i.e.,
\begin{align}
\label{eq:approximation}
    \rmA \rmA^\top \approx \mSigma.
\end{align}
\citet{li2022hyperdqn} provide an argument in expectation, i.e., $\mathbb{E}[\rmA \rmA^\top] = \mSigma$, and \citet{osband2023epistemic} provide an argument of asymptotic convergence, i.e., $\rmA \rmA^\top \stackrel{a.s.}{\longrightarrow} \mSigma$ when $M \rightarrow \infty$. These statements do not justify the usefulness of hypermodels or ENN for uncertainty estimation. A high-probability non-asymptotic characterization of the approximation in \cref{eq:factor,eq:approximation} is necessary for rigorous justification of their usefulness. Unfortunately, such results are not known in the literature.

We now provide the first analysis using our proposed unified probability tool in \cref{prop:djl-unit-norm}. First, we state the standard covering argument on the sphere and the argument on computing the norm on the covering set.
\begin{lemma}[Covering number of a sphere]
\label{lem:covering-sphere-1}
There exists a set $\mathcal{C}_{\varepsilon} \subset \mathbb{S}^{d-1}$ with $\left|\mathcal{C}_{\varepsilon}\right| \leq(1+ 2 / \varepsilon)^d$ such that for all $x \in \mathbb{S}^{d-1}$ there exists a $y \in \mathcal{C}_{\varepsilon}$ with $\|x-y\|_2 \leq \varepsilon$.
\end{lemma}
\begin{lemma}[Computing spectral norm on a covering set]
\label{lem:spectral-norm-eps-net}
Let $\mA$ be a symmetric $d \times d$ matrix, and let $\mathcal{C}_\varepsilon$ be an $\varepsilon$-covering of $\mathbb{S}^{d-1}$ for some $\varepsilon \in (0, 1)$. Then,
\[
  \|\mA\| = \sup_{x \in \mathbb{S}^{d-1}} |x^\top \mA x| \le (1- 2\varepsilon)^{-1} \sup_{x \in \mathcal{C}_{\varepsilon}} |x^\top \mA x|.
\]
\end{lemma}

Now we state the result in covariance matrix factorization with the specific goal of approximating the quadratic form
\begin{align}
\label{eq:goal-approximation}
    (1- \varepsilon) x^\top \mSigma x \le x^\top \rmA {\rmA}^\top x \le (1+ \varepsilon) x^\top \mSigma x, \quad \forall x \in \mathcal{X},
\end{align}
where $\mathcal{X}$ might be some set of interest in applications, e.g., the action space in bandit problems or the state-action joint space in reinforcement learning. Notice that the approximation in \cref{eq:approximation}, i.e., $(1-\varepsilon) \mSigma \preceq \rmA \rmA^\top \preceq (1+ \varepsilon) \mSigma$, reduces to \cref{eq:goal-approximation} when the set $\mathcal{X}$ is a compact set, e.g., $\{x\in \mathbb{R}^d: \|x\| = 1 \}$.
\begin{proposition}
\label{prop:statistics-computation-tradeoff}
\Cref{eq:goal-approximation} holds with probability at least $1 - \delta$ for the compact set $\mathcal{X} := \{x\in \mathbb{R}^d: \|x\| = 1 \}$ if $M \ge 64 \varepsilon^{-2} (d \log 9 + \log (2/\delta))$; for a finite set $\mathcal{X}$, if $M \ge 64\varepsilon^{-2} \log (2 |\mathcal{X}| / \delta)$.
\end{proposition}
\begin{proof}
Let us denote the random matrix as
\[
  \rmZ^\top = (\rmZ_0^\top, \rvz_1, \ldots, \rvz_T) \in \mathbb{R}^{M \times (d+T)},
\]
and the data matrix as
\[
  \mX = (\mSigma_0^{-1/2}, x_1/\sigma, \ldots, x_T/\sigma)^\top \in \mathbb{R}^{(d+T) \times d}.
\]
Notice the inverse posterior covariance matrix is $\mSigma^{-1} = \mSigma_0^{-1} + (1/\sigma^2) \sum_{t=1}^T x_t x_t^\top = \mX^\top \mX$. Then, we can represent
\begin{align*}
  \rmA = \mSigma \left(\mSigma_0^{-1/2} \rmZ_0 + \frac{1}{\sigma} \sum_{t=1}^T x_t \rvz_t^\top \right)
  = \mSigma \mX^\top \rmZ.
\end{align*}
Then $\rmA \rmA^\top = \mSigma \mX^\top \rmZ \rmZ^\top \mX \mSigma$ and $\mSigma = \mSigma \mX^\top \mX \mSigma$. The $(\varepsilon, \delta)$-approximation goal in \cref{eq:goal-approximation} reduces to a random projection argument with projection matrix $\rmZ^\top \in \mathbb{R}^{M \times (d+ T)}$ and the vector $\mX \mSigma x$ to be projected:
\begin{align}
\label{eq:eps-delta}
(1- \varepsilon) \|\mX \mSigma x\|^2 \le \|\rmZ^\top \mX \mSigma x\|^2 \le (1+ \varepsilon) \|\mX \mSigma x\|^2, \quad \forall x \in \mathcal{X}.
\end{align}
For the compact set $\mathcal{X} = \mathbb{S}^{d-1} = \{x \in \mathbb{R}^d: \|x\| = 1 \}$, by standard covering argument in \cref{lem:spectral-norm-eps-net} and \cref{prop:djl-unit-norm}, \cref{eq:eps-delta} holds with probability $1- \delta$ when $M \ge 64\varepsilon^{-2} (d \log 9 + \log (2 / \delta))$. For a finite set $\mathcal{X}$, direct application of the union bound with \cref{prop:djl-unit-norm} yields the result.
\end{proof}

\section{Conclusion}
\label{sec:conclu}
This study marks a pivotal advancement in dimensionality reduction research by offering a simple and unified framework for the Johnson-Lindenstrauss lemma. Our streamlined approach not only makes the lemma more accessible but also broadens its application across various data-intensive fields, including a pioneering validation of spherical construction for uncertainty estimation and reinforcement learning. The simplification of the theoretical underpinnings, alongside the unification of multiple constructions under a single analytical lens, represents a significant contribution to both the academic and practical realms. 

Through the extension of the Hanson-Wright inequality, providing precise constants for high-dimensional scenarios, and the introduction of novel probabilistic and analytical methods, we reinforce the JL lemma's indispensable role in navigating the complexities of high-dimensional data. This work underscores the power of simple, unified analyses in driving forward the understanding and application of fundamental concepts in computational algorithms and beyond, highlighting the direct pathway for future extensions and adaptations of random projection and Johnson-Lindentrauss.

\acks{The author would like to thank Professor Zhi-Quan (Tom) Luo for advising this project and Jiancong Xiao for helpful comments on the manuscript.}

\clearpage
\bibliography{ref}

\clearpage
\appendix

\crefalias{section}{appendix} %

\section{Non-negative diagonal extension for high-dimensional Hanson-Wright}
\label{sec:nnd-hdhw}

\begin{theorem}[High-dimensional Hanson-Wright with non-negative diagonal]
    \label{thm:hdhw-diag}
      Let $X_1, \ldots, X_n$ be independent, mean zero random vectors in $\R^m$, each $X_i$ is $K_i$-subGaussian. Let $K = \max_{i} K_i$.
      Let $A = (a_{ij})$ be an $n \times n$ matrix such that $a_{ii} \ge 0$.
      There exists a universal constant $C>0$ such that for any $t \ge 0$, we have
      \[
        \prob\left({ \abs{ \sum_{i, j =1}^n a_{ij} \langle X_i, X_j \rangle} \ge t }\right) \le \exp \left( - C \min \left\{ \frac{t^2}{ m K^4 \norm{A}^2_{F}}, \frac{t}{K^2 \norm{A}_{2}} \right\} \right).
      \]
\end{theorem}
\begin{proof}
    Decompose $\sum_{1 \leq i , j \leq n} a_{i j}\left\langle X_i, X_j\right\rangle=\sum_{i=1}^n a_{i i}\left\|X_i\right\|^2+S$, where $S=\sum_{1 \leq i \neq j \leq n} a_{i j}\left\langle X_i, X_j\right\rangle$. In view of the off-diagonal sum bound for $S$ in \Cref{thm:hdhw}, it suffices to show the following inequality for the diagonal sum: for any $t>0$,
    \begin{align}
    \label{eq:hdhw-diag}
    \mathbb{P}\left(\sum_{i=1}^n a_{i i}\left\|X_i\right\|^2 \geq m \sum_{i=1}^n a_{i i} K_i^2 + t\right) 
    \leq \exp \left[-C \min \left(\frac{t^2}{m K^4\sum_{i=1}^n a_{i i}^2}, \frac{t}{K^2 \max _{1 \leq i \leq n} a_{i i}}\right)\right]
    \end{align}
    since $\sum_{i=1}^n a_{i i}^2 \leq\|A\|_{F}^2$ and $\bar{a}:=\max _{1 \leq i \leq n} a_{i i} \leq\|A\|_{2}$. By Markov's inequality and \Cref{lemma:upper-squared-norm}, we have for any $\lambda>0$ and $t>0$,
    \begin{align*}
    \mathbb{P}\left(\sum_{i=1}^n a_{i i}\left(\left\|X_i\right\|^2-mK_i^2 \right) \geq t\right) & \leq e^{-\lambda t} \prod_{i=1}^n \mathbb{E}\left[e^{\lambda a_{i i}\left(\left\|X_i\right\|^2- m K_i^2 \right)}\right] \\
    & \leq e^{-\lambda t} \prod_{i=1}^n e^{2 \lambda^2 a_{i i}^2 m K_i^4} \\
    & \leq \exp \left(-\lambda t+2 \lambda^2 m \left(\sum_{i=1}^n a_{i i}^2\right) K^4\right)
    \end{align*}
    holds for all $0 \leq \lambda<\left(4 K^2 \bar{a}\right)^{-1}$. Choosing
    \[
    \lambda=\frac{t}{4\left(\sum_{i=1}^n a_{i i}^2\right) m K^4  } \wedge \frac{1}{8 \bar{a} K^2\|\Gamma\|_{2}},
    \]
    we get \cref{eq:hdhw-diag}.
\end{proof}

\begin{lemma}[Gaussianization for squared norm of a $\sigma$-sub-Gaussian random variable in $\mathbb{R}^n$]
\label{lem:mgf-squared-norm}
Let $X$ be a random variable in $\mathbb{R}^n$ such that $\mathbb{E}[X]=0$ and $\mathbb{E}[e^{z^\top X}] \leq \exp(\sigma^2 \norm{z}^2 / 2)$ for all $z \in \mathbb{R}^{n}$.
Let $Z \sim N(0, \sigma^2 I)$.
Then,
\[
\mathbb{E}\left[\exp{\frac{t\|X\|_2^2}{2}} \right] \leq \mathbb{E}\left[\exp{\frac{t\|Z\|_2^2}{2}} \right], \quad \forall 0 \leq t<\sigma^{-2} .
\]
\end{lemma}

\begin{proof}
The case for $t=0$ is obvious. 
Consider $t \in (0,\sigma^{-2} )$. Observe that
\begin{align*}
A & :=\frac{1}{(2 \pi)^{n / 2}\sigma^{n}} \int_{\mathbb{R}^n} \exp\left({-\frac{\norm{z}^2}{2 t}}\right) \mathbb{E}\left[\exp{z^\top X}\right] d z \\
& \stackrel{(1)}{=} \mathbb{E}\left[\frac{1}{(2 \pi)^{n / 2}\sigma^n } \int_{\mathbb{R}^n} \exp\left( {-\frac{\|z-t X\|_2^2}{2 t}} \right) d z  \exp \left( {\frac{t\|X\|_2^2}{2}} \right) \right] \\
& \stackrel{(2)}{=} \mathbb{E}\left[\exp \left( {\frac{t\|X\|_2^2}{2}} \right) \right]
\frac{1}{(2 \pi)^{n / 2}\sigma^n} \int_{\mathbb{R}^n} \exp \left( {-\frac{\|z\|_2^2}{2 t}} \right) d z \\
& \stackrel{(3)}{=} \mathbb{E}\left[\exp \left( {\frac{t\|X\|_2^2}{2}} \right) \right] \frac{1}{t^{-n/2} \sigma^n},
\end{align*}
where (1) follows from Fubini's theorem, (2) from the translational invariance of the Gaussian density integral, and (3) from that the integration of the standard Gaussian distribution $N(0, I_n)$ equals to one (requires $t > 0$). Thus, we get
\[
\mathbb{E}\left[\exp \left( {\frac{t\|X\|_2^2}{2}} \right) \right] = t^{-n/2} \sigma^n A .
\]

Since $\mathbb{E}\left[\exp {z^T X}\right] \leq \exp ({\sigma^2 \norm{z}^2/ 2} )$ for all $z \in \mathbb{R}^n$, we have for $t \in\left(0,\sigma^{-2}\right)$,
\begin{align*}
A & \leq \frac{1}{(2 \pi)^{n / 2} \sigma^n} \int_{\mathbb{R}^n} e^{-\frac{\norm{z}^2}{2 t}} e^{\frac{\sigma^2 \norm{z}^2}{2}} d z \\
& =\frac{1}{(2 \pi)^{n / 2}\sigma^n} \int_{\mathbb{R}^n} e^{-\frac{1}{2} \left(t^{-1} - \sigma^2\right) \norm{z}^2 } d z \\
& =\frac{1}{\sigma^n (t^{-1} - \sigma^2)^{n/2} } .
\end{align*}

Then we have
\[
\mathbb{E}\left[e^{\frac{t\|X\|_2^2}{2}}\right] \leq \frac{ t^{-n/2} \sigma^n }{\sigma^n (t^{-1} - \sigma^2)^{n/2} }=\frac{1}{(1- \sigma^2t)^{n/2}} \quad \forall 0 \leq t < \sigma^{-2} .
\]
On the other hand, for $Z \sim N(0, \sigma^2 I_n)$, similar calculations show that
\begin{align*}
\mathbb{E}\left[ e^{\frac{s\|Z\|_2^2}{2}} \right] & =\frac{1}{(2 \pi)^{n / 2}\sigma^n} \int_{\mathbb{R}^n} e^{-\frac{1}{2} \sigma^{-2} \norm{z}^2} e^{\frac{s}{2} \norm{z}^2} d z \\
& =\frac{1}{(2 \pi)^{n / 2}\sigma^n} \int_{\mathbb{R}^n} e^{-\frac{1}{2} (\sigma^{-2} - s) \norm{z}^2 } d z \\
& =\frac{1}{ (1-\sigma^2 s)^{n/2} } \quad \forall s<\sigma^{-2}.
\end{align*}
\end{proof}
\begin{remark}
    \Cref{lem:mgf-squared-norm} is true only for the upper tail as it requires $t \ge 0$. Without imposing additional assumptions, we cannot expect a lower tail bound for sub-Gaussian random variables as discussed in \citep{radoslaw2015note}.
\end{remark}
\begin{lemma}[Upper bound for MGF of squared norm of a $\sigma$-sub-Gaussian random variable in $\mathbb{R}^n$]
\label{lemma:upper-squared-norm}
In the setting of \cref{lem:mgf-squared-norm}, we have
\begin{align}
\label{eq:mgf-squared-norm-1}
    \mathbb{E}\left[\exp{ \left( \frac{t}{2}\left(\|X\|_2^2-n \sigma^2 \right) \right) }\right] \leq \exp{ \left( \frac{t^2}{2} (n \sigma^4 ) \right) } \quad \forall 0 \leq t< (2\sigma^2)^{-1}.
\end{align}
Consequently, we have for any $u>0$,
\begin{align}
\label{eq:prob-squared-norm}
    \mathbb{P}\left(\|X\|_2^2- n \sigma^2 \geq u \right) \leq \exp \left[-\frac{1}{8} \min \left(\frac{u^2}{ n \sigma^4 }, \frac{u}{ \sigma^2 }\right)\right] .
\end{align}
\end{lemma}
\begin{proof}
Let $Z \sim N(0, \sigma^2 I_n)$. 
By the calculations in \cref{lem:mgf-squared-norm}, we have for all $t<\sigma^{-2}$,
\[
\mathbb{E}\left[e^{\frac{t}{2}\left(\|Z\|_2^2-n\sigma^2\right)}\right]=\frac{e^{-\frac{t}{2} n\sigma^2}}{(1-\sigma^2 t)^{n/2}}
=
\left( \frac{e^{-t \sigma^2 / 2}}{\sqrt{1- \sigma^2 t}} \right)^n,
\]
Using the inequality
\[
\frac{e^{-t}}{\sqrt{1-2 t}} \leq e^{2 t^2} \quad \forall|t|<1 / 4,
\]
we have
\[
\mathbb{E}\left[e^{\frac{t}{2}\left(\|Z\|_2^2-n\sigma^2\right)}\right] \leq \exp( - t^2 \sigma^4/2 ) \quad \forall|t|<(2\sigma^2)^{-1} .
\]
Combining the last inequality with \cref{lem:mgf-squared-norm}, we get \cref{eq:mgf-squared-norm-1}. 

By Markov's inequality, we have for any $u>0$ and $0 \leq t<\left( 2 \sigma^2 \right)^{-1}$,
\[
\mathbb{P}\left(\|X\|_2^2-n\sigma^2 \geq u\right) \leq e^{-\frac{t u}{2}+\frac{t^2 \sigma^4 }{2} } .
\]
Choosing $t=t^*:=\frac{u}{2 n\sigma^4} \wedge \frac{1}{2\sigma^2}$, we get
\[
\mathbb{P}\left(\|X\|_2^2-n\sigma^2 \geq u\right) \leq \exp \left(-\frac{u t^*}{4}\right)=\exp \left[-\frac{1}{8} \min \left(\frac{u^2}{n \sigma^4}, \frac{u}{ \sigma^2 }\right)\right] .
\]
\end{proof}

\section{Proof of Generalized high-dimensional Hanson-Wright in \Cref{thm:hdhw-gen}}
\label{sec:general-hdhw}

\begin{proof}
  We prove the one-side inequality and the other side is similar by replacing $A$ with $-A$.
  Let
  \begin{align}
    \label{eq:def-S-g}
    S = \sum_{i, j: i \neq j}^n a_{ij} \langle b_i \odot X_i, b_j \odot X_j \rangle.
  \end{align}
  \textbf{Step 1: decoupling.}
  Let $\iota_1, \ldots, \iota_d \in \{0, 1\}$ be symmetric Bernoulli random variables, (i.e., $\prob(\iota_i = 0) = \prob(\iota_i = 1) = 1/2$) that are independent of $X_1, \ldots, X_n$.
  Since
  \[
    \E{\iota_i (1 - \iota_i)} =
    \begin{cases}
      0,   & i = j,    \\
      1/4, & i \neq j,
    \end{cases}
  \]
  we have $S = 4 \E[\iota]{S_{\iota} } $, where
  \[
    S_\iota = \sum_{i, j = 1}^n \iota_i(1-\iota_j) a_{ij} \langle b_i \odot X_i, b_j \odot X_j \rangle
  \]
  and the expectation $\E[\iota]{\cdot}$ is the expectation taken with respect to the random variables $\iota_i$.
  By Jensen's inequality, we have
  \[
    \E{\exp (\lambda S) } \le \E[X, \iota]{ \exp (4 \lambda S_{\iota}) }.
  \]
  Let $\Lambda_{\iota} = \{ i \in [d]: \iota_i = 1  \}$. Then we write
  \[
    S_{\iota} = \sum_{i \in \Lambda_{\iota}} \sum_{ j \in \Lambda_{\iota}^c } a_{ij} \langle b_i \odot X_i, b_j \odot X_j \rangle = \sum_{j \in \Lambda_{\iota}^{c} } \langle \sum_{i \in \Lambda_{\iota} } a_{ij} b_i \odot b_j \odot X_i, X_j \rangle.
  \]
  Taking expectation over $(X_j)_{j \in \Lambda_{\iota}^c } $ (i.e., conditioning on $(\iota_i)_{i= 1, \ldots, d}$ and $(X_i)_{i \in \Lambda_{\iota}}$ ), it follows that
  \[
    \E[(X_j)_{j \in \Lambda_{\iota}^c }]{
      \exp 4 \lambda S_{\iota}
    }
    = \prod_{j \in \Lambda_{\iota}^c} \E[(X_j)_{j \in \Lambda_{\iota}^c }]{
      e^ {\lambda \langle \sum_{i \in \Lambda_{\iota} } a_{ij} b_i \odot b_j \odot X_i, X_j \rangle}
    }
  \]
  by the independence among $(X_j)_{j\in \Lambda_{\iota} }$.
  By the assumption that $X_j$ are independent sub-Gaussian with mean zero, we have
  \[
    \E[(X_j)_{j \in \Lambda_{\iota}^c }]{
      \exp 4 \lambda S_{\iota}
    }
    \le \exp \left( \sum_{j \in \Lambda_{\iota}^c} 8 \lambda^2 K_j^2 \norm{ \sum_{i \in \Lambda_{\iota} } a_{ij} b_i \odot b_j \odot X_i }^2 \right) =: \exp \left( 8 \lambda^2 \sigma_{\iota}^2 \right).
  \]
  Thus we get
  \[
    \E[X]{ \exp (4 \lambda S_{\iota}) } \le \E[X]{ \exp (8 \lambda^2 \sigma_{\iota}^2) }.
  \]
  \textbf{Step 2: reduction to Gaussian random variables.}
  For $j=1, \ldots, n$, let $g_j$ be independent $N\left(0,16 K_j^2 \mI \right)$ random variables in $\R^m$ that are independent of $X_1, \ldots, X_n$ and $\iota_1, \ldots, \iota_n$. Define
  \[
    T:=\sum_{j \in \Lambda_\iota^c}\langle g_j, \sum_{i \in \Lambda_\iota} a_{i j} b_i \odot b_j \odot X_i \rangle.
  \]
  Then, by the definition of Gaussian random variables in $\R^m$, we have
  \[
    \begin{aligned}
      \E[g]{\exp{(\lambda T)} }
       & = \prod_{j \in \Lambda_\iota^c} \E[g]{ e^{\langle g_j, \lambda \sum_{i \in \Lambda_\iota} a_{i j} b_i \odot b_j \odot X_i \rangle} }  \\
       & =\exp \left(8 \lambda^2 \sum_{j \in \Lambda_\iota^c} K_j^2 \norm{ \sum_{i \in \Lambda_\iota} a_{i j} b_i \odot b_j \odot X_i }^2 \right)
      = \exp \left(8 \lambda^2 \sigma_\iota^2\right)
    \end{aligned}
  \]
  So it follows that
  \[
    \E[X]{\exp {(4 \lambda S_\iota)} } \leq \E[{X, g}]{ \exp {(\lambda T)} }.
  \]
  Since $T = \sum_{i \in \Lambda_\iota} \langle\sum_{j \in \Lambda_\iota^c} a_{i j} b_i \odot b_j \odot g_j, X_i \rangle$, by the assumption that $X_i$ are independent sub-Gaussian with mean zero, we have
  \[
    \E[\left(X_i\right)_{i \in \Lambda_\iota}]
    {\exp {(\lambda T)}}
    \leq \exp \left(\frac{\lambda^2}{2} \sum_{i \in \Lambda_\iota} K_i^2 \norm{ \sum_{j \in \Lambda_\iota^c} a_{i j} b_i \odot b_j \odot g_j}^2 \right),
  \]
  which implies that
  \begin{align}
    \label{eq:reduce-gaussian-g}
    \E[X]{\exp {(4 \lambda S_\iota)} } \le \E[g]{ \exp \left(\lambda^2 \tau_\iota^2 / 2\right) }
  \end{align}
  where $\tau_\iota^2=\sum_{i \in \Lambda_\iota} K_i^2 \norm{ \sum_{j \in \Lambda_\iota^c} a_{i j} b_i \odot b_j \odot g_j}^2$. Note that $\tau_\iota^2$ is a random variable that depends on $(\iota_i)_{i=1}^d$ and $(g_j)_{j=1}^n$.

  \textbf{Step 3: diagonalization.}
  We have $ g_j=\sum_{k=1}^{m} \left\langle g_j, e_k\right\rangle e_k$ and
  \begin{align*}
    \tau_\iota^2
     & =\sum_{i \in \Lambda_\iota} K_i^2\left\|\sum_{j \in \Lambda_\iota^c} a_{i j} b_i \odot b_j \odot g_j\right\|^2=\sum_{i \in \Lambda_\iota} K_i^2\left\|\sum_{k=1}^{m} \left( \sum_{j \in \Lambda_\iota^c} a_{i j} \left\langle b_i \odot b_j \odot g_j, e_k \right\rangle\right) e_k\right\|^2 \\
     & = \sum_{k=1}^{m} \sum_{i \in \Lambda_\iota}\left(\sum_{j \in \Lambda_\iota^c} K_i a_{i j} b_{ik} b_{jk} \left\langle g_j, e_k\right\rangle\right)^2                                                                                                                  \\
     & = \sum_{k=1}^m \norm{P_{\iota}\tilde{A} (I - P_\iota) G_k }^2
  \end{align*}
  where the last second step follows from Parseval's identity.
  $G_{j k}:=\left\langle g_j, e_k\right\rangle, j=1, \ldots, n$, are independent $N\left(0,16 K_j^2\right)$ random variables.
  $G_k = (G_{1k}, \ldots, G_{nk} )^\top \in \R^n$.
  $\widetilde{A}_k=\left(\tilde{a}_{i j} b_{ik} b_{jk} \right)_{i, j=1}^n$ with $\tilde{a}_{i j}=K_i a_{i j}$.
  Let $P_\iota \in \R^{n \times n}$ be the restriction matrix such that $P_{\iota, ii}=1$ if $i \in \Lambda_\iota$ and $P_{\iota, ij}=0$ otherwise.

  Define normal random variables $Z_k = (Z_{1k}, \ldots, Z_{nk})^\top \sim N(0,I)$ for each $k = 1, \ldots, m$.
  Then we have $G_k \stackrel{D}{=} \Gamma^{1/2} Z_k$ where $\Gamma = 16 \diag({ K_1^2}, \ldots, {K_n^2})$.

  Let $\tilde{A}_{\iota, k} := P_{\iota}\tilde{A}_k(I - P_{\iota})$.
  Then by the rotational invariance of Gaussian distributions, we have
  \begin{align*}
    \sum_{k=1}^m \norm{\tilde{A}_{\iota, k} G_k}^2 \stackrel{D}{=} \sum_{k=1}^m \norm{\tilde{A}_{\iota, k} \Gamma^{1/2} Z_k}^2  \stackrel{D}{=} \sum_{k=1}^m \sum_{j=1}^n s_{j, k}^2 Z_{jk}^2
  \end{align*}
  where $s_{jk}^2, j = 1, 2, \ldots, n$ are the eigenvalues of $\Gamma^{1/2} \tilde{A}_{\iota, k}^\top \tilde{A}_{\iota, k} \Gamma^{1/2}$ for each $k = 1, \ldots, m$.

  \textbf{Step 4: bound the eigenvalues.}
  It follows that
  \[
    \max _{j \in [n]} s_{j, k}^2 = \norm{\tilde{A}_{\iota, k} \Gamma^{1/2} }^2_2 \le 16 K^4 \norm{A^b_k}_2^2.
  \]
  In addition, we also have
  \begin{align*}
    \sum_{j=1}^n s_{jk}^2 & = \tr ( \Gamma^{1/2} \tilde{A}_{\iota,k}^\top \tilde{A}_{\iota,k} \Gamma^{1/2} ) \le 16 K^4 \| A^b_k \|_{F}^2
  \end{align*}
  and $\sum_{k=1}^m \sum_{j=1}^n s_{jk}^2 \le 16 K^4 \sum_{k = 1}^m \norm{A^b_k}_{F}^2$.
  Invoking \cref{eq:reduce-gaussian-g}, we get
  \[
    \mathbb{E}_X\left[e^{4 \lambda S_\iota}\right]
    \leq \prod_{k=1}^{m} \prod_{j=1}^n \mathbb{E}_Z \left[\exp \left(\lambda^2 s_{jk}^2 Z_{jk}^2 / 2\right)\right]
  \]
  Since $Z_{jk}^2$ are i.i.d. $\chi_1^2$ random variables with the moment generating function $\E{ e^{t Z_{jk}^2} } = (1 - 2 t)^{-1 / 2}$ for $t < 1 / 2$, we have
  \[
    \mathbb{E}_X\left[e^{4 \lambda S_\iota}\right] \leq \prod_{k=1}^{m} \prod_{j=1}^n \frac{1}{\sqrt{1-\lambda^2 s_{jk}^2}} \quad \text { if } \max_{j,k} \lambda^2 s_{jk}^2 < 1.
  \]
  Using $(1-z)^{-1 / 2} \leq e^z$ for $z \in[0,1 / 2]$, we get that if $\lambda^2 \max_{j,k}  s_{jk}^2 \le 1/2$, i.e., $ 32 K^4 \max_{k} \norm{A^b_k}_2^2 \lambda^2<1$, then
  \[
    \mathbb{E}_X\left[e^{4 \lambda S_\iota}\right] \leq \exp \left(\lambda^2 \sum_{k=1}^{m} \sum_{j = 1}^n s_{jk}^2\right) \leq \exp \left(16 \lambda^2 K^4 \sum_{k=1}^m \|A^b_k \|_{F}^2\right) .
  \]
  Note that the last inequality is uniform in $\iota$. Taking expectation with respect to $\delta$, we obtain that
  \[
    \mathbb{E}_X\left[e^{\lambda S}\right] \leq \mathbb{E}_{X, \iota}\left[e^{4 \lambda S_\iota}\right] \leq \exp \left(16 \lambda^2 K^4 \sum_{k=1}^m \|A^b_k \|_{F}^2\right)
  \]
  whenever $\abs{\lambda} < (4\sqrt{2} K^2 \max_k \norm{A^b_k}_2 )^{-1}$.

  \textbf{Step 5: Conclusion.} Now we have
  \[
    \mathbb{P}(S \geq t) \leq \exp \left(-\lambda t+16 \lambda^2 K^4 \sum_{k=1}^m \|A^b_k \|_{F}^2\right) \quad \text { for } \abs{\lambda} \leq \left(4 \sqrt{2} K^2 \max_{k} \norm{A^b_k}_2 \right)^{-1}.
  \]
  Optimizing in $\lambda$, we deduce that there exists a universal constant $C>0$ such that
  \[
    \mathbb{P}(S \geq t) \leq \exp \left[- \min \left(\frac{t^2}{64 K^4 \sum_{k=1}^m \|A^b_k\|_{F}^2}, \frac{t}{ 8 \sqrt{2} K^2 \max_{k} \norm{A^b_k}_2}\right)\right].
  \]
\end{proof}

\end{document}